\documentclass[letterpaper, 10 pt, conference]{IEEEtran}

\IEEEoverridecommandlockouts

\newcommand{\inertia}{\textrm{in}}

\renewcommand{\inertia}{\textsf{inertia}}

\usepackage{subcaption}
\usepackage{graphicx} 
\usepackage{comment}
\usepackage{overpic}
\usepackage{cite}
\usepackage{algorithm}
\usepackage{algorithmic}
\usepackage{amsthm}
\usepackage{amsmath} 
\usepackage{amssymb}
\newtheorem{theorem}{Theorem}

\newtheorem{definition}{Definition}
\usepackage{bm}
\usepackage{balance}
\usepackage{hyperref}
\hypersetup{
    colorlinks=true,
    linkcolor=black,
    filecolor=blue,      
    urlcolor=black,
    pdftitle={Inertia-Corrected Game},
    }

\usepackage{xcolor}
\floatstyle{plaintop}
\restylefloat{table}
\usepackage[tableposition=top]{caption}

\title{\LARGE \bf
Inertia-Corrected Newton Method For Generalized Nash Equilibria in Dynamic Games with Optimality Verification
}

\author{Zhiyuan Zhang$^{1}$ and Panagiotis Tsiotras$^{2}$
\thanks{This work is sponsored by ONR awards N00014-23-2308 and N00014-23-2353 and NSF award IIS-2008686}
\thanks{$^{1}$ School of Aerospace Engineering, Institute for Robotics and Intelligent Machines, Georgia Institute of Technology, Atlanta, GA 30332, USA, Email:
        {\tt\small zzhang615@gatech.edu}}%
\thanks{$^{2}$ School of Aerospace Engineering, Institute for Robotics and Intelligent Machines, Georgia Institute of Technology, Atlanta, GA 30332, USA, Email:
        {\tt\small tsiotras@gatech.edu}}%
}

\begin{document}
\maketitle
\thispagestyle{empty}
\pagestyle{empty}

\begin{abstract}
Newton methods efficiently find Generalized Nash Equilibria (GNE) in dynamic games by solving for the KKT necessary conditions.
These methods are fast and can support multi-agent Model Predictive Control (MPC) for highly dynamic robots.
However, a small KKT residual alone does not certify that the returned solution satisfies the
second-order sufficient conditions for a local GNE.
In this paper, we propose an efficient numerical method to verify the second-order sufficient conditions (SOSC) for a local GNE.
We connect the inertia of the agent KKT matrix with the positive definiteness of the reduced Hessian of the cost function, projected onto the null space of the constraints.
Furthermore, we introduce an inertia-corrected update step that improves convergence to local GNEs by destabilizing strict saddle points with weak cross-agent coupling.
Our main contribution is a fast Newton solver for Constrained Dynamic Games that provides efficient optimality checking.
Through numerical benchmarks, we demonstrate the proposed solver's runtime and convergence performance in practical multi-agent planning problems.
We also validate the solver's real-time capabilities in physical experiments using a platform of miniature autonomous race cars.
\end{abstract}

\section{Introduction}

\IEEEPARstart{I}{n} most multi-agent planning problems, the agents' objectives are intertwined.
For such problems, traditional single-agent optimal control methods are insufficient, as they typically assume a static environment and fail to account for the reactive nature of other agents \cite{gt_1, fridovich2020efficient}.
Discrete-time dynamic games provide a robust framework for multi-agent motion planning with coupled cost functions.
Similar to Model Predictive Control (MPC), dynamic games evolve over multiple stages, where the state at each stage depends on the control inputs from the previous stage.
A physics-based, time-discretized dynamics model governs this state evolution.
Many multi-robot planning problems can be modeled as dynamic games, including highway merging \cite{schwarting2019social, algames}, drone racing \cite{gt_1}, and social navigation \cite{ilqgame}.

A common solution concept for constrained dynamic games is the Generalized Nash Equilibrium (GNE).
A GNE is a set of control strategies for all agents such that no agent can unilaterally improve its cost without violating shared constraints.
Real-time numerical GNE solvers are fundamental for enabling MPC-style receding horizon control in interactive multi-agent environments.

In recent years, Newton methods based on KKT conditions have emerged as a promising direction
for solving dynamic games subject to constraints.
They can handle coupled state constraints and find GNE in practical dynamic games in sub-second timeframes \cite{algames, rd3g}. These advances open the door for real-time game-theoretic control.

However, the KKT conditions are only first-order necessary conditions.
Solutions satisfying KKT conditions may, in fact, be saddle points rather than local minima.
Despite this drawback, first-order solutions are often accepted in practice due to their simplicity and the minimal computational effort required.
Solvers such as ALGAMES \cite{algames} and RD3G \cite{rd3g} have demonstrated successful real-time receding-horizon control implementations, although the saddle-point vulnerability remains a well-known trade-off for speed.
While certain game structures, such as those with convex cost functions and decoupled constraints, guarantee that a KKT point is a true GNE \cite{rosen1965existence}, this property does not generally hold for many practical robotic games, which are highly nonlinear and non-convex.

In this paper, we present an efficient Newton solver for GNE that explicitly verifies the local sufficiency conditions.
To the best of our knowledge, this is the first work of its kind to provide such verification for \textit{constrained} dynamic games.
In addition, our solver uses an inertia-corrected descent step that destabilizes 
strict saddle points with sufficiently weak cross-agent coupling.
As shown in our numerical benchmarks, this key behavior improves robustness against saddle points compared to standard KKT solvers.

\section{Related Work}

\subsection{DDP-based Methods}
Dynamic-programming approaches extend Differential Dynamic Programming (DDP) to multi-agent games.
For example, iLQGames~\cite{ilqgame} repeatedly approximates a general-sum game by a linear-quadratic game, computes a local feedback solution through a backward pass, and updates the nominal trajectory through a forward rollout.
The standard iLQGames formulation is unconstrained, although constraint-compatible extensions have been proposed~\cite{ilqgame_barrier1}.
These methods target first-order stationarity without additional convexity assumptions, a stationary solution need not be a local Nash equilibrium.

\subsection{Newton-based Methods}

Newton methods, instead, solve the joint KKT conditions directly.
Multiple-shooting formulations use the states as decision variables and enforce the dynamics as equality constraints, enabling the use of sparse linear algebra solvers across the planning horizon.
ALGAMES~\cite{algames} combines this sparse structure with an augmented-Lagrangian treatment of inequalities and demonstrated faster solution times than iLQGames in autonomous-driving benchmarks.
RD3G~\cite{rd3g} uses an active-set formulation and demonstrated a 29-Hz receding-horizon controller on two autonomous race cars.
Despite their practical speed, these solvers terminate based on first-order KKT residuals and do not certify the per-agent second-order conditions.

\subsection{Saddle Points and Second-Order Verification}

Game updates can be viewed as a dynamical system whose fixed points include both local equilibria and non-equilibrium stationary points.
Local Symplectic Surgery~\cite{lss} for zero-sum games and Double Follow-the-Ridge~\cite{dftr} for general-sum games modify the solver dynamics to avoid undesirable stationary points.
These methods, however, are not designed to exploit the sparse constrained multiple-shooting structure, and are therefore not fast enough for real-time robotic planning.
The present work addresses this gap by combining a sparse Newton solver with efficient per-agent SOSC verification and an inertia-corrected update that destabilizes a class of strict saddle points.

\section{Problem Formulation}

Let $i \in \{1, \dots, N\}$ denote the agent index in an $N$-agent game. For a problem with horizon $T$, let $x^i_k \in \mathbb{R}^n$, $k=0,\dots,T$, denote the state of agent $i$, and let $u^i_k \in \mathbb{R}^m$, $k=0,\dots,T-1$, denote its control input.
The initial state $x^i_0$ is fixed.
Thus, the primal decision vector contains $x^i_1,\dots,x^i_T$ and $u^i_0,\dots,u^i_{T-1}$.
For notational brevity, we adopt the following stacked vector convention:
$x_k$ denotes the concatenated states of all agents at stage $k$;
$x^i$ denotes the trajectory of agent $i$ over all stages;
and $x$ denotes the joint trajectory of all agents over the entire horizon.
Similar stacking logic applies to the control variables $u$ and the multipliers.

The Generalized Nash Equilibrium Problem (GNEP) with separable dynamics is formulated as:
\begin{subequations}\label{eq:game}
\begin{align}
&\min_{x^i,u^i} J^i(x,u^i),\quad i = 1, \dots, N, \\
\textrm{s.t. } & x^i_{k+1} = f(x^i_k, u^i_k), \quad k=0,\dots,T-1,  \label{eq:dyn_con}\\
& h(x_k) \leq 0, \quad k=1,\dots,T , \label{eq:ineq_con}
\end{align}
where
\begin{align}
& J^i(x,u^i) = \sum_{k=0}^{T-1} J^i_k(x_k, u^i_k) + J^i_T(x_T).
\end{align}
\end{subequations}
Here, $f(\cdot)$ represents the dynamics, $h(\cdot)$ are the shared state constraints (e.g., collision avoidance), and $J^i(\cdot)$ is the cost function for agent $i$.
For simplicity, we assume homogeneous dynamics and identical constraints for all agents.
However, our analysis supports heterogeneous dynamics and constraints.

An (open-loop) joint control strategy $u^* = (u^{1*}, \dots, u^{N*})$ is a local 
Generalized Nash Equilibrium (GNE) if, for every agent $i$, there exists a neighborhood $\mathcal{N}^i$ around $u^{i*}$ such that
\begin{equation}\label{eq:gne}
J^i(x(u^{i*}, u^{-i*}), u^{i*}) \leq J^i(x(u^{i}, u^{-i*}), u^{i})
\end{equation}
for all $u^i \in \mathcal{N}^i \cap \Omega^i(u^{-i*})$, where $\Omega^i(u^{-i*})$ is the set of feasible controls for agent $i$ given the fixed strategies of other agents $u^{-i*}$.

\section{Methodology}

\subsection{Residual Descent}
We define the Lagrangian for agent $i$ in the GNEP (\ref{eq:game}) as
\begin{align}\label{eq:lagrangian}
    \mathcal{L}^i &= \sum_{k=0}^{T-1} J^i_k(x_k, u^i_k) + J^i_T(x_T)
    + \sum_{k=1}^{T} (\mu^i_k)^\top h(x_k)\nonumber \\
    &\quad + \sum_{k=0}^{T-1} (\lambda^{i}_k)^\top [f(x_k^i, u_k^i) - x_{k+1}^i] ,
\end{align}
where $\lambda$ and $\mu$ are the dual variables (e.g., multipliers) for the dynamics and inequality constraints, respectively.
Let $\mathcal{F}^i(x^i, u^i) = 0$ represent the vectorized dynamics constraints across the horizon, and let $\mathcal{H}(x) \leq 0$ represent the concatenated inequality constraints.
The KKT necessary conditions~\cite{gnep} for agent $i$ are:
\begin{subequations}\label{eq:kkt}
\begin{align}
\nabla_{x^i} \mathcal{L}^i &= 0, \quad \nabla_{u^i} \mathcal{L}^i = 0, \\
\mathcal{F}^i(x^i, u^i) &= 0, \label{eq:fi}\\
\mathcal{H}(x) &\leq 0, \label{eq:h}\\
\mu^i \geq 0, \quad & (\mu^i)^\top \mathcal{H}(x) = 0. \label{eq:comp_slack}
\end{align}
\end{subequations}
We define the KKT residual vector for agent $i$ as $r^i = (\nabla_{x^i}\mathcal{L}^i, \nabla_{u^i}\mathcal{L}^i, \mathcal{F}^i, \mathcal{H}_{\text{active}})$, and the joint residual as $r = (r^1, \dots, r^N)$.
We employ an active-set method; that is, only constraints considered active are included in the residual computation via 
$\mathcal{H}_{\text{active}}$.
Furthermore,  strict complementarity is assumed: if an inequality constraint is active, its corresponding multiplier $\mu$ is strictly positive.
To improve convergence robustness and provide gradient information for the inactive constraints, we augment the cost function of each agent with a standard logarithmic barrier term~\cite{opt_textbook}. 
Specifically, for the inactive constraints $k \notin \mathcal{A}$, the agent cost $J^i$ is augmented with a term of the form
    $ - \tau \sum_{k \notin \mathcal{A}} \log(-\mathcal{H}_k(x)),$
where $\tau > 0$ is the barrier parameter, and where $\mathcal{A}$ denotes the set of active constraints.
While the active-set strategy explicitly handles boundary enforcement via Lagrange multipliers, the barrier terms on the inactive constraints act as a ``repulsive potential,'' shaping the cost landscape to guide the solver away from the feasible boundary during the descent and preventing aggressive steps from violating currently inactive constraints.

Defining the primal-dual decision variable as $y = (x, u, \lambda, \mu)$, the Newton descent step $\Delta y$ results from the solution to the linear system
\begin{equation}\label{eq:newton}
    \mathrm{D}_y r(y) \Delta y = -r(y),
\end{equation}
where $\mathrm{D}_y r(y)$ denotes the Jacobian of the joint KKT residual.
This system admits a unique solution when the matrix $D_y r(y)$ is nonsingular.

\subsection{Sufficiency Verification}

A solution satisfying the KKT conditions (\ref{eq:kkt}) is a candidate GNE, since the first-order conditions alone do not certify local optimality for each agent.
For constrained optimization, the Second-Order Sufficient Conditions (SOSC) require that the Hessian of the Lagrangian be positive definite on the critical cone~\cite{opt_textbook}.
Let $z^i=(x^i,u^i)$ denote the primal decision vector of agent $i$, and let $p^i$ denote a feasible first-order perturbation of $z^i$.
Under the above strict-complementarity assumption, the critical cone at a KKT point $z^*$ reduces to the tangent subspace
$\mathcal{C}^i(z^*)=\{p^i:A^i(z^*)p^i=0\}$, where $A^i$ is the Jacobian of the dynamics and active inequality constraints.
Let $H^i(z^*)=\nabla^2_{z^i z^i}\mathcal{L}^i(z^*)$ denote the corresponding Lagrangian Hessian.
The per-agent SOSC, and hence a sufficient condition for a strict local GNE, is
\begin{equation}\label{eq:sosc}
    (p^i)^\top H^i(z^*)p^i > 0, \quad
    \forall p^i \in \mathcal{C}^i(z^*) \setminus \{0\}, \quad i=1,\ldots,N.
\end{equation}
Checking condition \eqref{eq:sosc} directly requires computing a basis for the null space of the active constraints in (\ref{eq:h}) and (\ref{eq:fi}), 
which is computationally expensive for large, even sparse systems.
%
Instead, we adopt an approach to certify a candidate solution as a strict minimizer by checking the inertia of the KKT matrix.

\begin{definition}[Matrix Inertia] \label{def:inertia}
The inertia of a matrix is defined as the tuple $(n_+, n_-, n_0)$ representing the number of positive, negative, and zero eigenvalues.
\end{definition}

Consider the Equality Quadratic Programming (EQP) problem associated with the local quadratic approximation of the game:
\begin{equation}\label{eq:eqp}
\begin{split}
\min_{z^i} \quad&\frac12 (z^i)^\top H^i z^i + (g^i)^\top z^i, \\
\text{s.t.} &\quad A^i z^i = 0.
\end{split}
\end{equation}
Construct the KKT matrix $K^i$:
\begin{equation}
    K^i = \begin{bmatrix}
        H^i & (A^i)^\top \\
        A^i & 0 \\
    \end{bmatrix},
\end{equation}
where $A^i \in \mathbb{R}^{t \times d}$ is the Jacobian matrix of the active constraints for player $i$.
Specifically, its rows consist of the first-order partial derivatives of all active constraints with respect to player $i$'s decision variables.
Here, $t$ denotes the total number of active constraints for player $i$, and $d$ is the dimension of the $i$th agent's decision space.
If $\mathrm{rank}(A^i) = t$, that is, $A^i$ has full row rank,
%
%
\begin{theorem}[KKT Inertia Criterion]\label{th:in}
Following~\cite[Th.~2.1]{gould}, the EQP (\ref{eq:eqp}) has a unique minimizer if and only if:
\begin{equation}
    \inertia (K^i) = (d, t, 0).
\end{equation}
\end{theorem}
\medskip

In our context, $d = T(n+m)$ is the number of primal variables for agent $i$, and $t$ is the total number of its dynamics equalities and active inequality constraints.
Using 
Sylvester's Law of Inertia~\cite{matrix_computation}, one can employ a sparse $LDL^\top$ decomposition 
\begin{equation}
    K^i = L^i D^i (L^i)^\top,
\end{equation}
to efficiently verify the inertia of the matrix $K^i$.
By simply evaluating the signs of the diagonal elements of $D^i$ (pivots), this approach avoids the prohibitive cost of an eigenvalue decomposition while remaining numerically stable and preserving the sparsity of the dynamic game structure~\cite{benzi2005numerical}.

The factorization cost depends on the matrix dimension and the fill-in created by the sparse elimination ordering.
Let $s_i=d_i+t_i$ be the order of $K^i$, where $d_i=T(n+m)$ is the number of primal state and control variables for agent $i$, and $t_i$ is the number of its equality and currently active inequality constraints.
If column $j$ of the factor $L^i$ contains $\ell_j$ nonzero entries, the numerical sparse $LDL^\top$ factorization requires $\mathcal{O}(\sum_{j=1}^{s_i}\ell_j^2)$ operations and $\mathcal{O}(\operatorname{nnz}(L^i))$ memory, where $\operatorname{nnz}$ denotes the number of non-zero elements in a matrix. 
The subsequent triangular solves require $\mathcal{O}(\operatorname{nnz}(L^i))$ operations.
In the dense worst case, these bounds reduce to $\mathcal{O}(s_i^3)$ time and $\mathcal{O}(s_i^2)$ memory.
Once the factorization is available, determining the inertia from the pivots of $D^i$ requires only $\mathcal{O}(s_i)$ of additional work.



\subsection{Inertia Correction}\label{sec:ic}
Standard Newton methods may converge to stationary points that fail the per-agent SOSC.
To improve robustness, we introduce an inertia correction step to regularize the global game Jacobian as
\begin{equation} \label{eq:correct}
\tilde{D}_y r(y) = D_y r(y) + E,
\end{equation}
where $E$ is the regularization matrix.
The inertia-corrected Newton step (\ref{eq:newton}) is then computed using the regularized game Jacobian.
\begin{equation}\label{eq:ic_newton}
    \tilde{D}_y r(y) \Delta y = -r(y).
\end{equation}

We now develop the regularization matrix $E$.
If the inertia check $\inertia(K^i) \neq (d, t, 0)$ fails for any agent, we regularize the game Jacobian using a near-minimal feasible shift.
Specifically, let the agent-specific KKT matrix
\begin{equation}
K^i(\epsilon)=
\begin{bmatrix}
H^i+\epsilon I & (A^i)^\top\\
A^i & 0
\end{bmatrix},
\end{equation}
and define the threshold
\begin{equation}\label{eq:reg}
\epsilon_*^i := \inf\{\epsilon\geq0:
\inertia(K^i(\epsilon))=(d,t,0)\}.
\end{equation}
The binary search returns a feasible $\epsilon^i\leq\epsilon_*^i+\tau_\epsilon$ whose regularized KKT matrix has the target inertia.
Computing this near-minimal shift requires repeated inertia calculations for the regularized KKT matrix.
In our implementation, the $LDL^\top$ decomposition is separated into an expensive symbolic factorization and a comparatively inexpensive numerical factorization.
Since the regularization does not change the sparsity structure of the KKT matrix, only the numerical computation is repeated online.
This allows us to use a binary search to determine a near-minimal correction while reusing the symbolic factorization.

A common alternative to our minimal inertia correction method 
is a geometric strategy that repeatedly multiplies the regularization by a large factor (e.g., ten) until the target inertia is obtained.
Although this strategy is simple, the accepted shift can exceed the smallest sufficient value by almost an order of magnitude.
On the feasible tangent space, the corrected step is governed by the inverse of the reduced Hessian plus $\epsilon^i I$.
When $\epsilon^i$ is much larger than the local curvature, this inverse is approximately $(1/\epsilon^i)I$, so the primal update approaches a small projected-gradient step and loses much of the curvature scaling responsible for Newton's fast local progress~\cite{opt_textbook}.
The proposed near-minimal correction preserves more of the original curvature while still enforcing the required inertia.
Because each binary-search trial requires only a numerical $LDL^\top$ refactorization with the same sparsity pattern, this finer adjustment adds modest computational cost.
For a search interval $[\epsilon_{\mathrm{L}},\epsilon_{\mathrm{U}}]$ and termination tolerance $\tau_\epsilon$, at most
$B_i=\lceil\log_2((\epsilon_{\mathrm{U}}-\epsilon_{\mathrm{L}})/\tau_\epsilon)\rceil$
numerical factorizations are required for an agent that needs correction.

The inertia correction operation then applies a block-diagonal regularization matrix $E = \textrm{blockdiag}(E^1, \dots, E^N)$ as in~\eqref{eq:correct}, where each block is defined as $E^i = \textrm{diag} (\epsilon^i I, 0)$.
This modification is designed to repel the class of strict saddle points characterized below while preserving the local Newton step near a regular GNE.

\medskip

The proposed procedure is summarized in Algorithm~\ref{alg:inertia_correction}.

\begin{algorithm}
\caption{Inertia-Corrected Newton Method}
\label{alg:inertia_correction}
\begin{algorithmic}
\REQUIRE Initial guess $y$, residual tolerance $\tau$, correction tolerance $\tau_\epsilon$
\ENSURE Solution $y^*$, SOSC certificate $\sigma$
\REPEAT
    \STATE Update active set $\mathcal{A}$ based on current constraints
    \FOR{$i=1$ \TO $N$}
        \STATE Augment $J^i$ with barriers for inactive constraints
        \STATE Construct $H^i$ and $A^i$ for the active constraints
        \STATE $d_i\leftarrow\dim(z^i)$, $t_i\leftarrow\operatorname{rows}(A^i)$
        \STATE Compute Inertia: $(n_+, n_-, n_0) \leftarrow \inertia (K^i)$
        \IF{$(n_+,n_-,n_0)=(d_i,t_i,0)$}
            \STATE $\epsilon^i \leftarrow 0$
        \ELSE
            \STATE Determine a near-minimal feasible $\epsilon^i$ to correct \inertia$(K^i)$
        \ENDIF
    \ENDFOR
    \STATE Construct $\tilde{\mathrm{D}}_y r(y)$ using regularized $H^i + \epsilon^i I$
    \STATE Solve $\tilde{\mathrm{D}}_y r(y) \Delta y = -r(y)$
    \STATE $y \leftarrow y + \Delta y$
\UNTIL{$\|r(y)\| < \tau$}
\STATE Update active set and construct each unregularized $K^i(y)$
\STATE $\sigma \leftarrow \text{True}$
\FOR{$i=1$ \TO $N$}
    \STATE $d_i\leftarrow\dim(z^i)$, $t_i\leftarrow\operatorname{rows}(A^i)$
    \IF{$\inertia(K^i(y))\neq(d_i,t_i,0)$}
        \STATE $\sigma \leftarrow \text{False}$
    \ENDIF
\ENDFOR
\RETURN $y, \sigma$
\end{algorithmic}
\end{algorithm}

\subsection{Local Convergence Near a GNE}
The following local result uses the standard regularity conditions for Newton methods applied to a KKT system. In particular, the per-agent conditions ensure that no inertia correction is needed near the solution, while nonsingularity of the joint game Jacobian accounts for cross-agent coupling.

\begin{theorem}[Local Convergence Near a GNE]\label{th:local_convergence}
Let $y^*$ be a local GNE satisfying $r(y^*)=0$, strict complementarity, and Linearly Independent Constraint Qualification (LICQ) 
for every agent.
Assume that the active set is fixed in a neighborhood of $y^*$.\footnote{For suitable active-set rules, finite identification of the solution's active constraints is a standard local result under nondegeneracy conditions~\cite{oberlin2006active}. The active set may change away from the solution; the theorem concerns the local iterations after this identification has occurred.} 
Assume also that the strong SOSC holds for every agent, that is,
\begin{equation}
    (p^i)^\top H^i(y^*)p^i>0,
    \quad \forall p^i\in\ker A^i(y^*)\setminus\{0\}.
\end{equation}
Furthermore, assume that the joint game Jacobian $D_y r(y^*)$ is nonsingular and that $D_y r$ is locally Lipschitz continuous. Then, for an initial point sufficiently close to $y^*$, the iterates generated by Algorithm~\ref{alg:inertia_correction} converge quadratically to $y^*$.
\end{theorem}

\begin{proof}
By LICQ and the strong SOSC, Theorem~\ref{th:in} yields
\begin{equation}
    \inertia(K^i(y^*))=(d_i,t_i,0)
\end{equation}
for every agent $i$, where $d_i$ is the dimension of its primal decision vector and $t_i$ is the number of equality and active inequality constraints in $A^i$.
Because the active set is locally fixed by assumption, the KKT matrix retains the same constraint blocks along the local iterates.
Since the problem derivatives are continuous, $K^i(y)$ is a continuous symmetric matrix-valued function near $y^*$. Its eigenvalues therefore remain bounded away from zero in a sufficiently small neighborhood of $y^*$, and its inertia remains $(d_i,t_i,0)$. Consequently, Algorithm~\ref{alg:inertia_correction} selects $\epsilon^i=0$ for every agent throughout this neighborhood.

The corrected game Jacobian therefore coincides locally with the exact Jacobian, $\tilde{D}_y r(y)=D_y r(y)$, and the proposed update reduces to the standard Newton iteration
\begin{equation}
    y^{(k+1)}=y^{(k)}-[D_y r(y^{(k)})]^{-1}r(y^{(k)}).
\end{equation}
Since $D_y r(y^*)$ is nonsingular and $D_y r$ is locally Lipschitz continuous, the standard local Newton convergence theorem yields quadratic convergence for all initial points sufficiently close to $y^*$~\cite{opt_textbook}.
\end{proof}

\subsection{Instability of a Class of Saddle Points}

We next show that inertia correction makes a nontrivial class of strict saddle points locally unstable. 
Unlike the preceding result in Theorem~\ref{th:local_convergence}, 
this statement concerns the stability of the actual discrete update rather than its continuous-time approximation.

Consider the damped update map
\begin{equation}\label{eq:update_map}
    \Phi_\alpha(y)
    =y-\alpha[\tilde{D}_y r(y)]^{-1}r(y),
    \quad 0<\alpha\leq1,
\end{equation}
where Algorithm~\ref{alg:inertia_correction} corresponds to $\alpha=1$. At a KKT point $\tilde y$, define
\begin{equation}
    D_y r(\tilde y)=K_{\mathrm{diag}}+C,
\end{equation}
where $K_{\mathrm{diag}}=\operatorname{blockdiag}(K^1,\ldots,K^N)$ contains the per-agent KKT matrices and $C$ contains the cross-agent derivatives.
More precisely, $C$ collects the off-diagonal blocks of the joint Jacobian: for $i\neq j$, the $(i,j)$ block contains the derivatives of agent $i$'s KKT residual $r^i$ with respect to agent $j$'s primal-dual variables $y^j$. These terms arise from coupled objectives and shared constraints and quantify how one agent's variables affect another agent's optimality conditions.
For some agent $i$, let $Z^i$ be the matrix having orthonormal columns spanning $\ker A^i$. 
Define the reduced Hessian as
\begin{equation}
    R^i=(Z^i)^\top H^i Z^i.
\end{equation}
We call a KKT point a strict saddle if at least one agent's reduced Hessian 
has a strictly negative eigenvalue; such a point violates that agent's second-order necessary condition~\cite{opt_textbook}.

\begin{theorem}[Instability of Weakly Coupled Strict Saddles]
Let $\tilde y$ satisfy $r(\tilde y)=0$, and suppose the active set is locally fixed and the active constraints satisfy the Linearly Independent Constraint Qualification (LICQ).
Suppose that for some agent $i$, the reduced Hessian $R^i$
has a negative eigenvalue. 
Let $\epsilon^i>0$ be the inertia correction selected so that $R^i+\epsilon^i I\succ0$. If $C=0$, then $\tilde y$ is an unstable fixed point of $\Phi_\alpha$ for every $0<\alpha\leq1$. Moreover, this instability persists for all sufficiently small cross-agent coupling matrices $C$.
\end{theorem}

\begin{proof}
First consider $C=0$. Let $q$ be a unit eigenvector of $R^i$ associated with $\nu<0$. 
Since
\begin{equation}
    (Z^i)^\top\bigl(H^iZ^iq-\nu Z^iq\bigr)=0,
\end{equation}
the vector $H^iZ^iq-\nu Z^iq$ lies in $\operatorname{range}((A^i)^\top)$. Thus, there exists a multiplier perturbation $\eta$ such that
\begin{equation}
    (A^i)^\top\eta=\nu Z^iq-H^iZ^iq.
\end{equation}
Using $v^i=(Z^iq,\eta)$, the uncorrected and corrected KKT matrices satisfy the expressions
\begin{equation}
    K^iv^i=\begin{bmatrix}\nu Z^iq\\0\end{bmatrix},
    \qquad
    (K^i+E^i)v^i=\begin{bmatrix}(\nu+\epsilon^i)Z^iq\\0\end{bmatrix}.
\end{equation}
Therefore,
\begin{equation}
    (K^i+E^i)^{-1}K^iv^i
    =\frac{\nu}{\nu+\epsilon^i}v^i.
\end{equation}
Since $R^i+\epsilon^i I\succ0$, we have $\nu+\epsilon^i>0$. 
The Jacobian of the update map is therefore
\begin{equation}\label{eq:jac}
    D\Phi_\alpha(\tilde y)
    =I-\alpha[\tilde{D}_y r(\tilde y)]^{-1}D_y r(\tilde y).
\end{equation}
For $C=0$, the vector obtained by placing $v^i$ in the $i$th agent block and zeros elsewhere is an eigenvector of this Jacobian in (\ref{eq:jac})
with eigenvalue
\begin{equation}
    \rho=1-\alpha\frac{\nu}{\nu+\epsilon^i}>1.
\end{equation}
Hence, $\tilde y$ is an unstable fixed point.

Finally, the matrix $D\Phi_\alpha(\tilde y)$ depends continuously on $C$ whenever the corrected Jacobian is nonsingular. At $C=0$, the corrected Jacobian is block diagonal and nonsingular by construction, and it has an eigenvalue $\rho>1$. Nonsingularity and the eigenvalue outside the unit circle persist under sufficiently small perturbations~\cite{matrix_computation}. 
Thus, there exists $\delta>0$ such that $\tilde y$ remains unstable whenever $\lVert C\rVert<\delta$.
\end{proof}
The theorem covers strict saddles with a feasible negative-curvature direction when cross-agent coupling is absent or sufficiently weak. It does not assert global avoidance of every non-GNE KKT point: strong coupling can alter the spectrum of the joint update, and degenerate points with zero but no negative reduced curvature require a separate analysis. 
Nevertheless, the previous result establishes that inertia correction actively repels a well-defined class of saddle points, consistent with the behavior observed in the Numerical Simulations Section~\ref{sec:sim}.

\section{Numerical Simulations}\label{sec:sim}

We empirically evaluated the proposed solver across two distinct multi-agent dynamic game scenarios: highway merging and uncontrolled intersection traversal.
We compared the runtime of our solver, both with and without inertia correction, against the established ALGAMES~\cite{algames} and iLQGames~\cite{ilqgame} baselines.
We also studied the effect of inertia correction on the optimality of the solutions returned by our solver.

\subsection{Implementation Details}

The proposed solver utilizes a hybrid architecture. The core algorithm is implemented in C++ for improved performance, while the problem definition and high-level interface are provided in Python.
Cost functions and dynamics derivatives are generated using CasADi, and low-level linear algebra and active-set logic are handled by Eigen3.
All numerical benchmarks were conducted on a workstation equipped with an Intel i7-7700k (4.2\,GHz) CPU and 32\,GB of RAM, running Ubuntu 22.04.
Computations are single-threaded to provide a fair baseline comparison.
All solvers terminate at the same first-order KKT residual tolerance, in iLQGames, collision avoidance is represented using soft constraints.
\footnote{ The source code repository will be released with the final version of the manuscript.}

Both simulation benchmarks use a horizon of $T=20$ with four states and two controls per vehicle, the initial states are fixed and are not optimization variables.
Thus, an $N$-vehicle instance contains $120N$ primal decision variables and $80N$ dynamics equality constraints.
Each vehicle is represented by two collision circles, yielding four pairwise circle-separation inequalities per vehicle pair and stage, or $40N(N-1)$ collision inequalities over the horizon.
The merging and intersection benchmarks therefore have the same dimensions: for $N=2,\ldots,8$, the respective (decision-variable, total-constraint) counts are $(240,240)$, $(360,480)$, $(480,800)$, $(600,1200)$, $(720,1680)$, $(840,2240)$, and $(960,2880)$.
Only the currently active subset of the collision inequalities is included in each Newton KKT system.
The quadratic growth of the pairwise constraints and the associated factorization fill-in ultimately limit the centralized formulation to a moderate number of strongly interacting agents.
Large robot swarms would generally require a localized interaction model, distributed solution method, or hierarchical information structure, and this analysis is outside the intended scope of the present solver.

\subsection{Highway Merging Scenario}

Highway merging is a standard benchmark for dynamic games involving competitive behavior.
The scenario initializes $N$ vehicles with random longitudinal positions and velocities across two adjacent lanes.
The objective for all agents is to merge into the left lane while maximizing forward progress.
Each vehicle is modeled using the kinematic bicycle model with state $x=[p_x, p_y, v, \theta]^\top$ and control $u=[a,\delta]^\top$, representing longitudinal position, lateral position, speed, heading, acceleration, and steering angle, respectively.
The stage cost for agent $i$ is defined as:
\begin{equation}
J^i_k(x_k,u^i_k) = \| x_k^i - x^i_{\textrm{ref}} \|_{Q_r}^2 + \| u_k^i \|_{R}^2 - \gamma \sum_{j \neq i} p_{x,k}^j,
\end{equation}
where the first term penalizes deviations from the target lane and reference speed, and the second term regularizes control effort.
The third term acts as a ``progress reward,'' incentivizing agent $i$ to overtake or maintain a lead relative to other agents $j$.
This coupling in the cost function makes the merging scenario a general-sum game.

Collision avoidance is enforced via coupled inequality constraints.
Each vehicle is approximated by two bounding circles, a constraint is violated if the distance between any pair of circles from different agents falls below a safety threshold.
While the cost function is quadratic, the non-convex collision constraints create a challenging optimization landscape with multiple local minima.

\begin{figure}[t!]
  \centering
  \begin{subfigure}{\linewidth}
  \centering
  \includegraphics[clip, trim=300 52 280 52,angle=-90,width=0.85\linewidth]{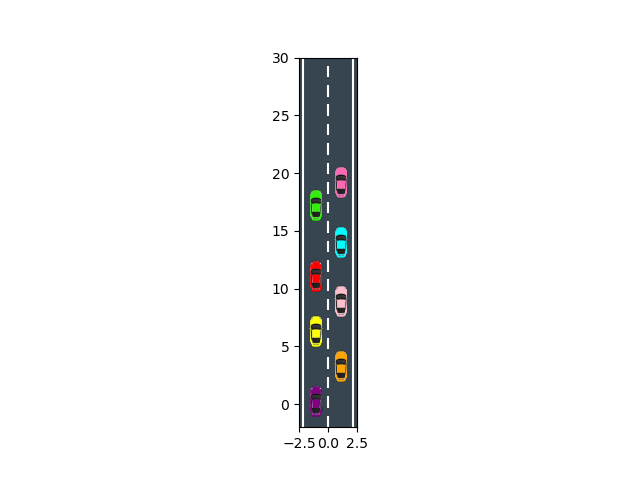}
    \caption{Initial State ($k=0$)}
  \end{subfigure}
  \begin{subfigure}{\linewidth}
  \centering
  \includegraphics[clip, trim=300 52 280 52,angle=-90,width=0.85\linewidth]{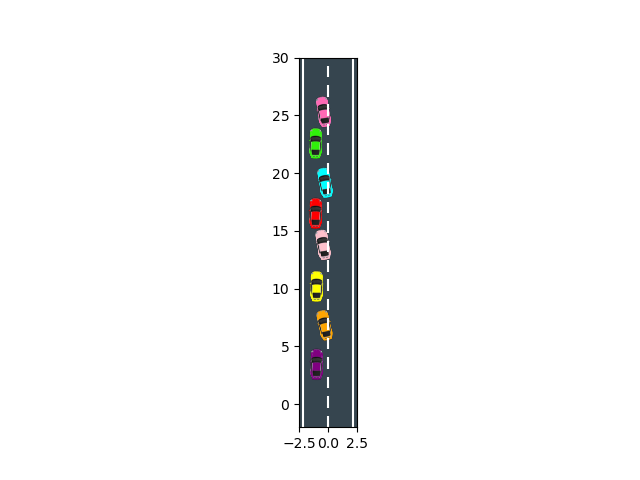}
    \caption{Intermediate State ($k=T/2$)}
  \end{subfigure}
  \begin{subfigure}{\linewidth}
  \centering
  \includegraphics[clip, trim=300 52 280 52,angle=-90,width=0.85\linewidth]{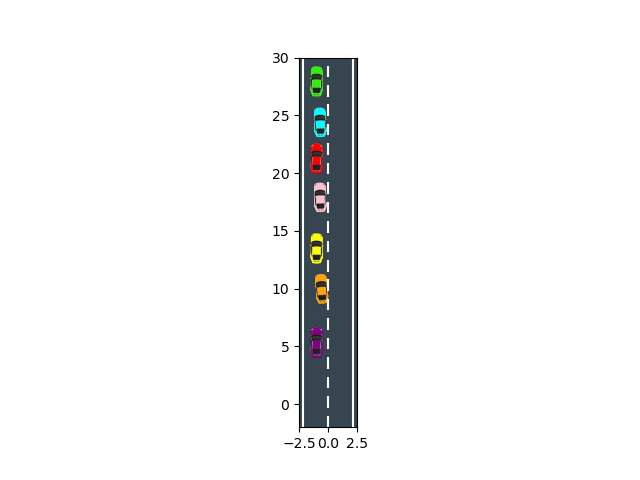}
    \caption{Final State ($k=T$)}
  \end{subfigure}
  \caption{Time-lapse of a solution to the 8-car merging game. Agents successfully negotiate the merge while maintaining safety distances.}
  \label{fig:merge}
\end{figure}

\subsection{Uncontrolled Intersection Scenario}

In this scenario, $N$ vehicles are initialized with random positions and velocities approaching a four-way intersection.
The objective is to traverse the intersection while maintaining the current lane and reference speed without collisions.
The dynamics and collision constraints are identical to the highway merging case.
The cost function is defined as:
\begin{equation}
J^i_k(x_k,u^i_k) = \| x_k^i - x^i_{\textrm{ref}} \|_{Q_r}^2 + \| u_k^i \|_{R}^2.
\end{equation}
Unlike the merging scenario, this cost structure does not include direct coupling terms (e.g., progress rewards).
However, the problem remains a Generalized Nash Equilibrium Problem (GNEP) due to the shared, non-convex collision constraints.
The intersection scenario typically exhibits more complex interaction patterns than merging, as agents must negotiate right-of-way from orthogonal directions, often requiring significant control deviations to avoid collisions.

\begin{figure}[t]
  \centering
  \begin{subfigure}[b]{0.3\linewidth}
  \includegraphics[trim=140px 50 120 60,clip,width=\linewidth]{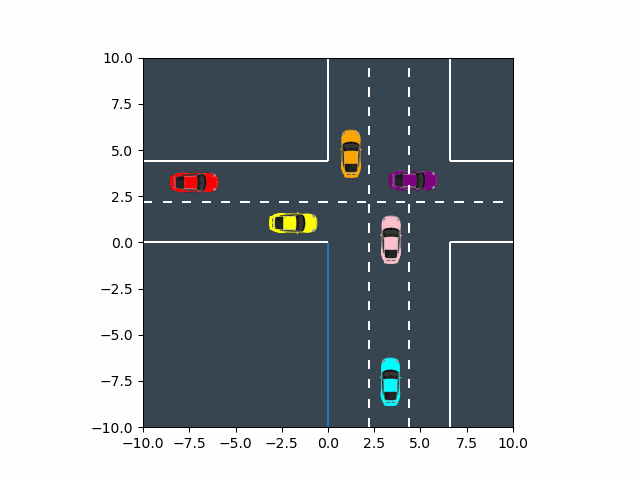}
    \caption{Initial State }
  \end{subfigure}
  \begin{subfigure}[b]{0.3\linewidth}
  \includegraphics[trim=140px 50 120 60,clip,width=\linewidth]{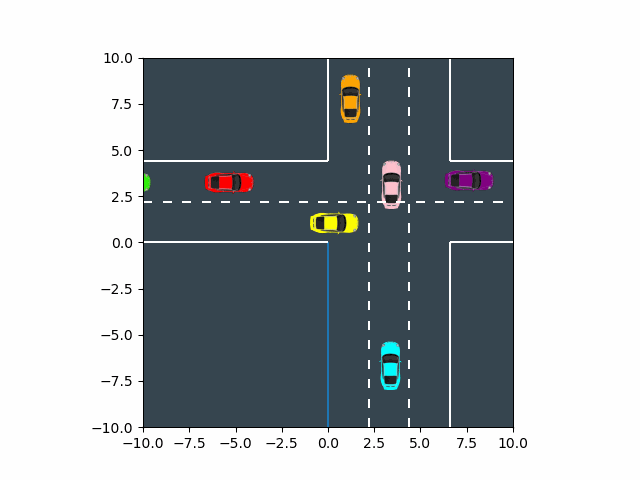}
    \caption{Interim State }
  \end{subfigure}
  \begin{subfigure}[b]{0.3\linewidth}
  \includegraphics[trim=140px 50 120 60,clip,width=\linewidth]{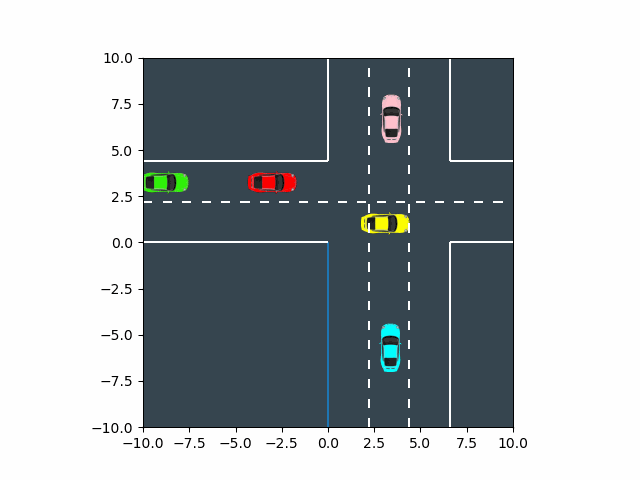}
    \caption{Final State }
  \end{subfigure}
  \caption{Snapshots of the 6-car intersection game.}
  \label{fig:intersection}
\end{figure}

\subsection{Benchmark Results}

We evaluated the solver's performance by averaging results over 100 trials with randomized initial conditions.
The same random seed was used across all compared algorithms to ensure consistency.

\subsubsection{Optimality Verification}

We define the ``Optimality Rate'' as the fraction of KKT-converged solutions that also satisfy the per-agent Second-Order Sufficient Conditions (SOSC).
A low optimality rate therefore indicates that many returned KKT solutions cannot be certified as strict local Generalized Nash Equilibria (GNEs) by the SOSC test. 
It does not classify degenerate points that fail SOSC without exhibiting strict negative curvature.

Figure~\ref{fig:merge_conv} illustrates the impact of the proposed inertia correction mechanism.
Without correction, the Newton method frequently converges to KKT points that fail the SOSC test, particularly as the number of agents and interactions increases.
With inertia correction enabled, the solver modifies the Hessian when the target inertia is not attained, significantly increasing the likelihood of returning an SOSC-certified local GNE.
The lower single-start optimality rates in the more strongly coupled merging games also reflect sensitivity to the initial control guess, as expected for a local solver applied to a nonconvex problem.
To evaluate this effect, we repeat the inertia-corrected solver with up to ten independently randomized initial control guesses and retain the first solution satisfying both the KKT conditions and SOSC.
This restart strategy achieves an optimality rate of $100\%$ for 2--7 vehicles and $98\%$ for 8 vehicles.
The restart result does not provide a global convergence guarantee, but it shows that SOSC-satisfying solutions are consistently recoverable and that the initial basin of attraction is the primary limitation in these instances.

\begin{figure}[tb!]
    \centering
    \includegraphics[width=0.9\linewidth]{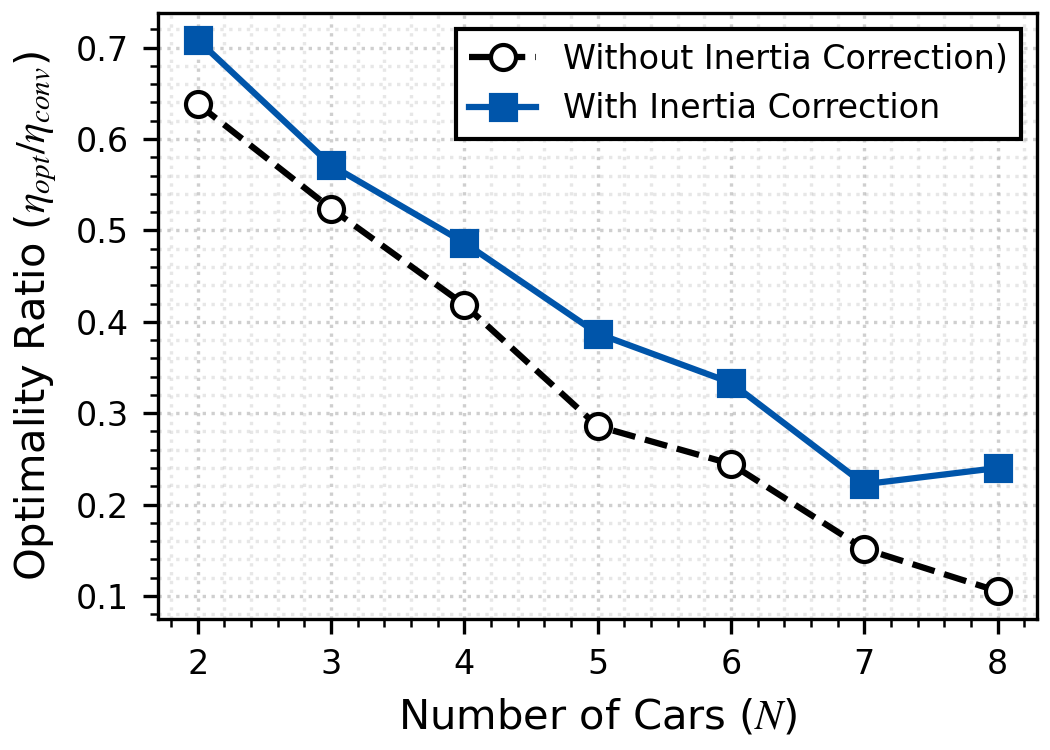}
    \caption{Single-start optimality rate in the merging game with and without inertia correction. The proposed correction significantly increases the fraction of KKT-converged solutions that satisfy the per-agent SOSC.}
    \label{fig:merge_conv}
\end{figure}

\subsubsection{Regularization Strategy}

We compared our binary-search correction with a geometric baseline that multiplies the regularization by a factor of ten until the KKT matrix attains the target inertia.
Figure~\ref{fig:reg_comparison} shows the average KKT residual at each solver iteration for representative 2-, 3-, 6-, and 7-car merging games, the subset is shown to keep the plot readable.
Both methods initially reduce the residual and then enter a transient plateau before terminal convergence.
The near-minimal correction generally exits this plateau earlier, most notably in the larger games, and reaches the terminal convergence regime in fewer iterations.
In contrast, the geometric baseline is more likely to over-regularize the reduced Hessian, suppressing useful curvature information and producing a smaller, more gradient-like step.
The rapid terminal decrease by several orders of magnitude is also consistent with the local quadratic convergence established in Theorem~\ref{th:local_convergence}, because the curves average multiple trials, they should be viewed as empirical supporting evidence rather than a direct rate measurement.

\begin{figure}[tb!]
    \centering
    \includegraphics[trim=10px 10 10 10,clip,width=0.95\linewidth]{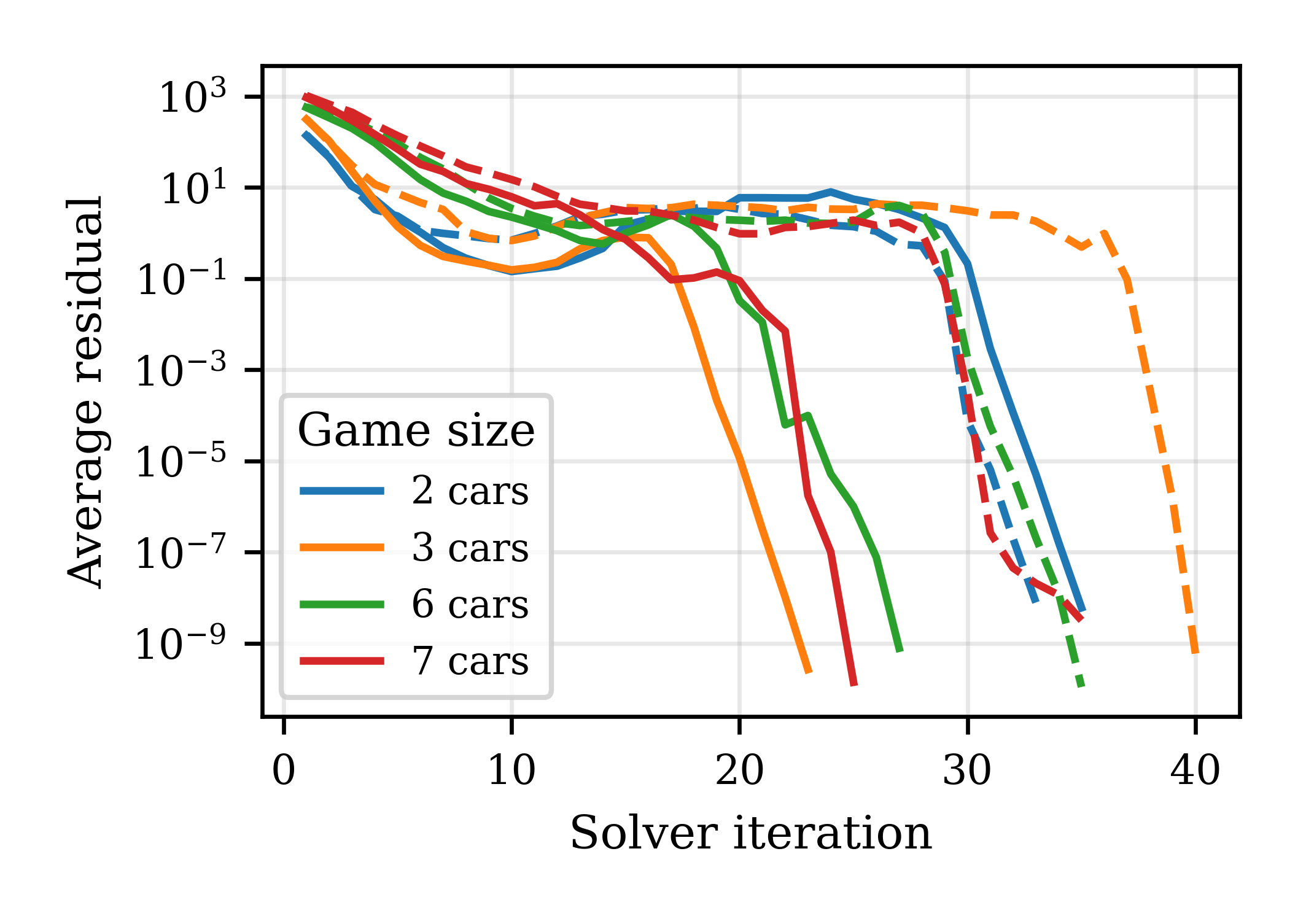}
    \caption{Average KKT residual versus solver iteration for selected merging games. Solid lines use the proposed near-minimal binary-search correction.
    Dashed lines use the geometric baseline that increases the regularization by a factor of ten. The near-minimal correction generally shortens the transient stagnation period.}
    \label{fig:reg_comparison}
\end{figure}

\subsubsection{Runtime Performance}

Figures~\ref{fig:merge_runtime} and \ref{fig:intersection_runtime} compare the runtime of the proposed solver, with and without inertia correction, against ALGAMES~\cite{algames} and iLQGames~\cite{ilqgame}.
Both variants of the proposed solver are significantly faster than the two baselines across all tested problem sizes and in both scenarios.
For the highway merging game, inertia correction introduces a modest runtime overhead relative to the uncorrected solver, but the corrected solver remains substantially faster than ALGAMES and iLQGames.
For the intersection game, the runtimes with and without inertia correction are nearly identical, indicating that optimality verification and correction can be incorporated at little additional computational cost in this scenario.
The difference in added cost likely arises from the sparser interaction in the intersection game, which results in a simpler $LDL^\top$ factorization.
The per-agent inertia checks are independent and can also be parallelized in performance-critical applications.

\begin{figure}[tb!]
    \centering
    \includegraphics[trim=0px 0 0 15,clip,width=0.9\linewidth]{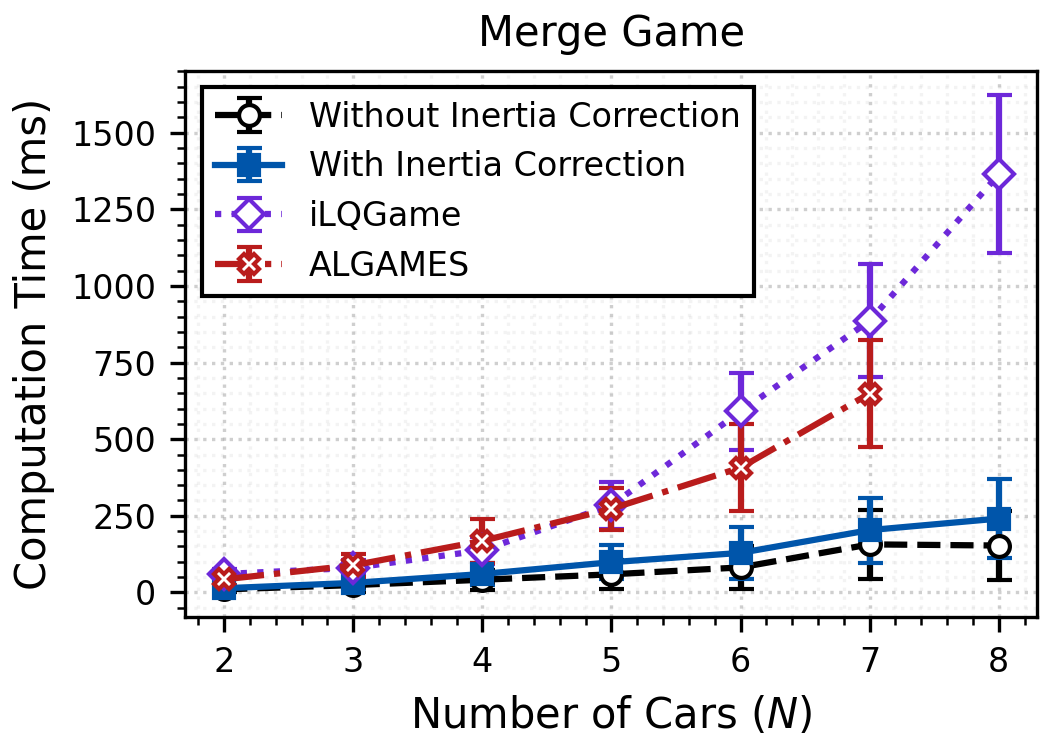}
    \caption{Computation time versus number of agents for the highway merging game. Both variants of the proposed solver are significantly faster than ALGAMES and iLQGames.}
    \label{fig:merge_runtime}
\end{figure}

\begin{figure}[tb!]
    \centering
    \includegraphics[trim=0px 0 0 15,clip,width=0.9\linewidth]{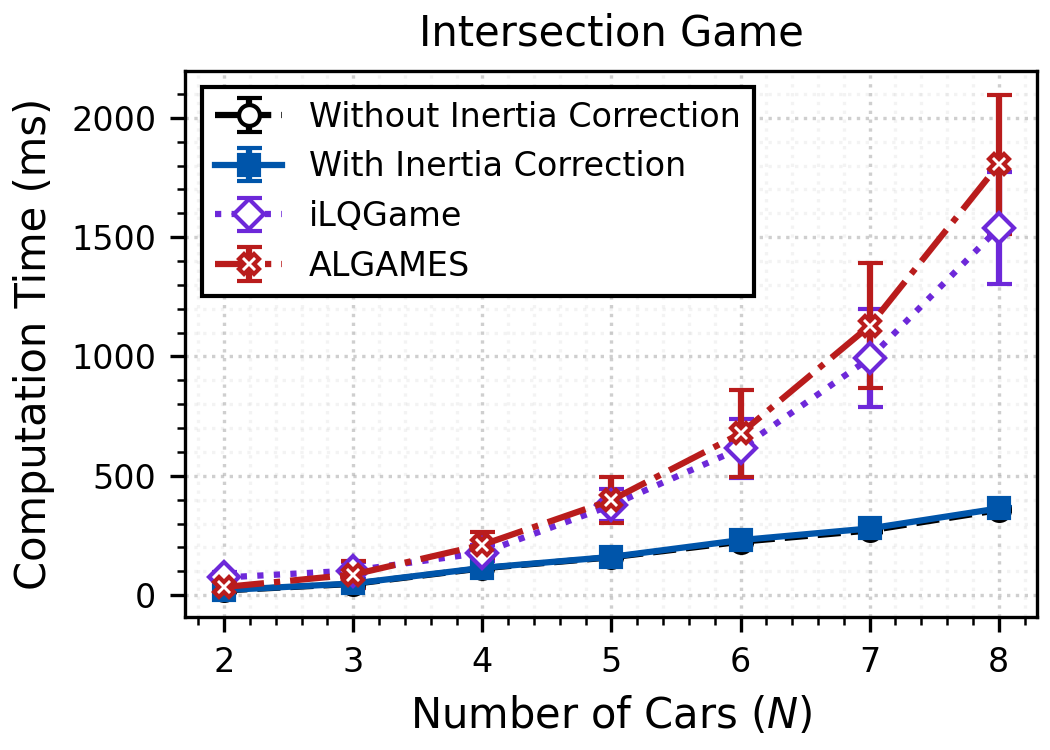}
    \caption{Computation time versus number of agents for the intersection game. Inertia correction adds negligible runtime overhead, and both variants of the proposed solver are significantly faster than ALGAMES and iLQGames.}
    \label{fig:intersection_runtime}
\end{figure}

\begin{figure}[tb!]
    \centering
    \includegraphics[width=0.9\linewidth]{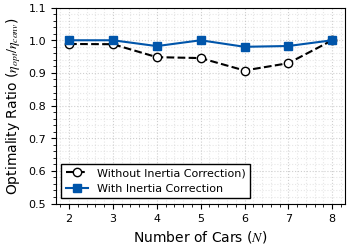}
    \caption{Single-start optimality rate (SOSC-certified solutions divided by KKT-converged solutions) in the intersection game.}
    \label{fig:intersection_conv}
\end{figure}

\section{Experimental Validation}\label{sec:exp}
To validate the real-world applicability of the proposed solver, we deployed it as a receding horizon planner on the BuzzRacer platform~\cite{buzzracer}. 
The hardware experiments feature a dynamic 4-car racing scenario governed by a hierarchical planning and control architecture. 
At the high level, the planner leverages the proposed game solver to compute a local Generalized Nash Equilibrium over a 20-step receding horizon, with a typical planning iteration converging in under 130 ms. 
These trajectories are subsequently passed to lower-level, vehicle-specific path-tracking controllers running at 100Hz.
The experiment was repeated multiple times, with the vehicles placed in various starting configurations along the track. 
Throughout the experiments, the vehicles navigated a tight track in close proximity to one another, avoiding collisions with both stationary boundary walls and dynamic opponents. 
During interactions, the autonomous agents exhibit aggressive, human-like driving behaviors such as strategic weaving, aggressive overtaking, and sustained side-by-side racing. 
Figure~\ref{fig:exp} shows key snapshots from these dynamic interactions.

\begin{figure}[t]
  \centering
  \begin{subfigure}[b]{0.3\linewidth}
  \includegraphics[trim=0px 0 0 30,clip,width=\linewidth]{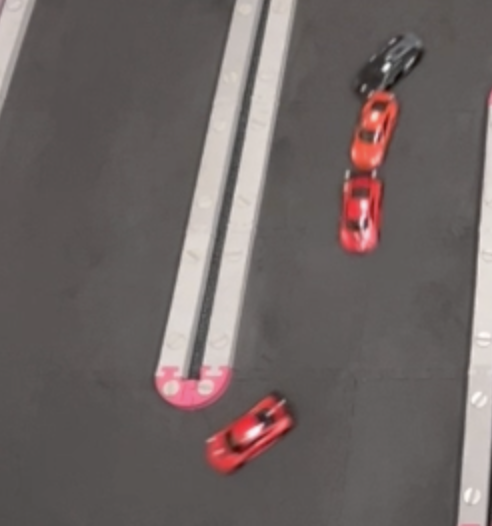}
  \end{subfigure}
  \begin{subfigure}[b]{0.3\linewidth}
  \includegraphics[trim=0px 10 0 30,clip,width=\linewidth]{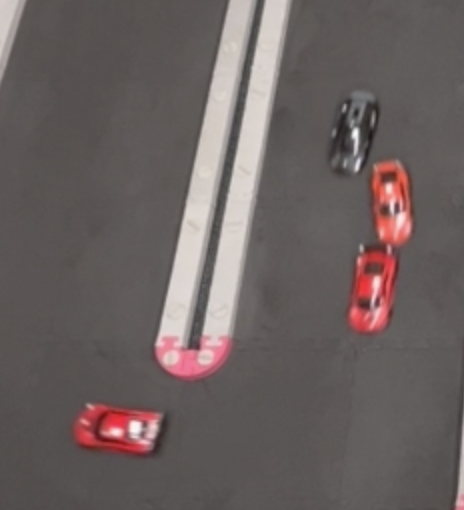}
  \end{subfigure}
  \begin{subfigure}[b]{0.3\linewidth}
  \includegraphics[trim=10px 0 0 10,clip,width=\linewidth]{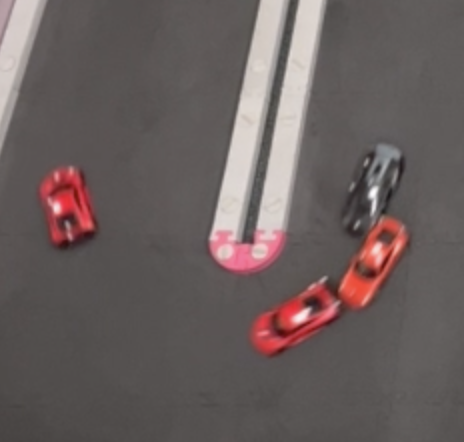}
  \end{subfigure}
  \\[1ex]
  \begin{subfigure}[b]{0.3\linewidth}
  \includegraphics[trim=0px 0 0 10,clip,width=\linewidth]{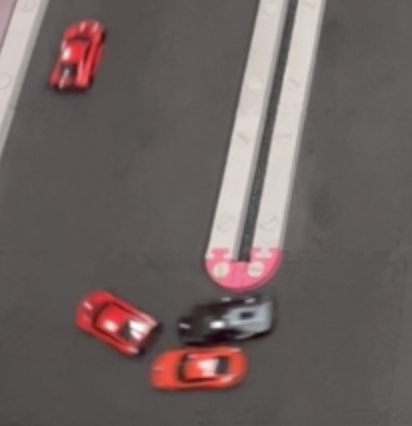}
  \end{subfigure}
  \begin{subfigure}[b]{0.3\linewidth}
  \includegraphics[trim=0px 0 0 0,clip,width=\linewidth]{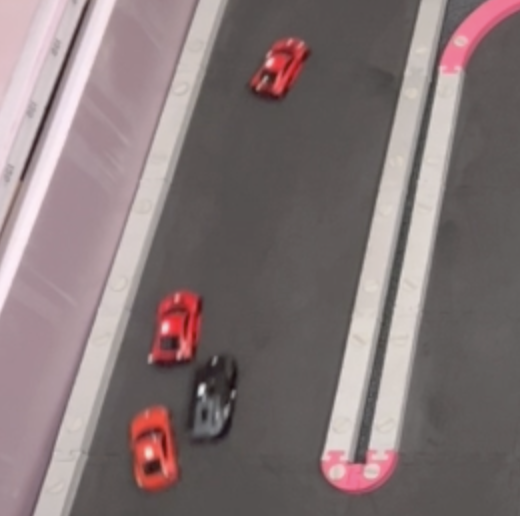}
  \end{subfigure}
  \begin{subfigure}[b]{0.3\linewidth}
  \includegraphics[trim=0px 0 50 0,clip,width=\linewidth]{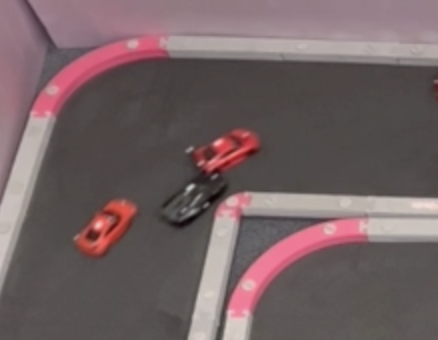}
  \end{subfigure}
  \caption{Snapshots of an overtake during the experiment.}\label{fig:exp}
\end{figure}

Supplementary video footage of the simulations and experiments is provided to further demonstrate the solver's capabilities.

\section{Conclusion}

We have presented an efficient Newton-based solver for Generalized Nash Equilibria in constrained dynamic games.
A key contribution of this work is the integration of an efficient SOSC certificate based on the inertia of the KKT matrix.
Unlike existing solvers that report only first-order convergence, our method identifies whether a returned KKT point satisfies sufficient conditions for a strict local GNE.
Furthermore, we introduced an inertia-correction mechanism that regularizes the descent direction, destabilizes weakly coupled strict saddle points, and empirically improves convergence to valid GNEs.
Numerical benchmarks demonstrate substantially lower runtime than the ALGAMES and iLQGames baselines while showing improved robustness in complex multi-agent scenarios.
Physical experiments on miniature autonomous race cars further confirm the solver's applicability for real-time robotic control.

\balance
\bibliographystyle{ieeetran}
\bibliography{refs}
\end{document}